\documentclass{aims}

\usepackage{algpseudocode,bbm}
\usepackage[plain]{algorithm}
\usepackage{amssymb, amsmath, amsthm, amscd, array}
\usepackage{balance, booktabs, bm}
\usepackage[labelsep=period]{caption}
\usepackage{calc, color, comment}
\usepackage{cases}
\usepackage{empheq, enumitem, epsfig, epstopdf, etoolbox}
\usepackage{flafter}
\usepackage{float}
\usepackage{geometry}
\usepackage{graphicx}
\usepackage[utf8]{inputenc}
\usepackage{latexsym, lipsum, lineno}
\usepackage{mathtools}
\usepackage{makecell, multirow}
\usepackage{pdflscape}
\usepackage{pict2e}
\usepackage{rotating}
\usepackage{setspace, subcaption}
\usepackage{tabularx}
\usepackage[flushleft]{threeparttable}
\usepackage{wrapfig}

\makeatletter
\renewcommand{\ALG@name}{Scheme}
\makeatother

\makeatletter
\renewcommand{\ALG@beginalgorithmic}{\footnotesize}
\makeatother

\makeatletter
\newcommand{\multiline}[1]{%
	\begin{tabularx}{\dimexpr\linewidth-\ALG@thistlm}[t]{@{}X@{}}
		#1
	\end{tabularx}
}
\makeatother

\numberwithin{equation}{section}

\modulolinenumbers[1]

\newtheorem{theorem}{Theorem}[section]
\theoremstyle{definition}
\newtheorem{remark}{Remark}

\usepackage{txfonts}

\def\DOI{10.3934/mbe.2024xxx}

\makeatletter
\renewcommand{\@biblabel}[1]{#1\hfill \hspace{-0.2cm}}
\makeatother

\usepackage{cite}

\numberwithin{equation}{section}

\usepackage{microtype}

\begin{document}

\title{Stochastic forest transition model dynamics and parameter estimation via deep learning}

\author{%
  Satoshi Kumabe\affil{1},
  Tianyu Song\affil{2}
  and
  T\^{o}n Vi\d {\^{e}}t T\d{a}\affil{1,2,3,}\corrauth
}

\shortauthors{the Author(s)}

\address{%
  \addr{\affilnum{1}}{Joint Graduate School of Mathematics for Innovation, Kyushu University, 744 Motooka Nishi Ward, Fukuoka 819-0395, Japan}
  \addr{\affilnum{2}}{Graduate School of Bioresource and Bioenvironmental Sciences, Kyushu University, 744 Motooka Nishi Ward, Fukuoka 819-0395, Japan}
  \addr{\affilnum{3}}{Center for Promotion of International Education and Research, Faculty of Agriculture, Kyushu University, 744 Motooka Nishi Ward, Fukuoka 819-0395, Japan}}
\corraddr{Email: tavietton[at]agr.kyushu-u.ac.jp.}

\begin{abstract}
Forest transitions, characterized by dynamic shifts between forest, agricultural, and abandoned lands, are complex phenomena. This study developed a stochastic differential equation model to capture the intricate dynamics of these transitions. We established the existence of global positive solutions for the model and conducted numerical analyses to assess the impact of model parameters on deforestation incentives. To address the challenge of parameter estimation, we proposed a novel deep learning approach that estimates all model parameters from a single sample containing time-series observations of forest and agricultural land proportions. This innovative approach  enables us to understand forest transition dynamics and deforestation trends at any future time.
\end{abstract}


\maketitle
\allowdisplaybreaks

\section{Introduction}
Forests are vital ecosystems that support biodiversity, regulate the climate, and provide numerous ecosystem services. However, global forest cover has been undergoing significant changes in recent decades, with the expansion of agriculture and urbanization leading to deforestation and forest degradation in many regions. Conversely, in certain areas, reforestation and forest recovery have been observed, indicating a dynamic process known as forest transition \cite{Meyfroidt}.

Forest transition refers to a change of land-use in a given territory from forest land loss to forest land recovery. This phenomenon has drawn significant attention from researchers, policymakers, and environmentalists due to its implications for environmental conservation, land use dynamics, and sustainable development \cite{Mather,rudel2005}. Understanding the patterns and dynamics of forest transition is crucial for formulating effective policies and strategies to foster sustainable land management and forest preservation.

A number of approaches have been proposed to understand forest transition dynamics, including empirical studies, behavioral models, and conceptual frameworks that highlight socio-economic feedbacks and policy impacts  \cite{Walker2004, Lambin2010}.  Notably, Satake and Rudel \cite{Satake} introduced a deterministic model to explore the individual landowner's incentives for deforestation. The model assumes that individual landowner decisions and landscape-level processes, such as forest regeneration from abandoned land, collectively shape the overall forest transition dynamics. Their model, a system of difference equations, is given by:
\begin{equation}\label{FTM}
	\left\{ \,
	\begin{split}
		&x_{n+1} =  p(x_n)(1-x_n-y_n) - r(x_n)x_n +x_n,  \\
		&y_{n+1} =  r(x_n)x_n - \eta y_n + y_n, 
	\end{split}
	\right.
\end{equation}
where
$x_n, y_n,$ and $1 - x_n- y_n$ represent the proportions of forest land, agricultural land, and abandoned land at time $n$, respectively (other parameters are detailed in Section \ref{Model_desription}).  
By analyzing the equilibrium and stability of this system, Satake and Rudel concluded that the rates of future discounting and forest regrowth are crucial factors influencing the likelihood of forest transition.

While deterministic models provide valuable insights, they often neglect the randomness inherent in ecological and socioeconomic systems. Real-world land-use change is influenced by unpredictable events such as forest fires, pest outbreaks, policy shifts, and market shocks \cite{rudel2020, goni2018}, introducing substantial stochasticity. To capture these effects, stochastic models based on stochastic differential equations (SDEs) have been increasingly used in ecological and epidemiological modeling \cite{Allen2003, Mao, Gao2025,Aditya2024}. These models provide a more realistic framework to study complex systems by incorporating random perturbations into biological or human-driven processes.


In this work, we propose a stochastic extension of the Satake–Rudel model that incorporates environmental and socioeconomic noise into the dynamics of forest transition. The inclusion of stochasticity allows us to examine how random shocks can influence long-term land-use trends and the economic incentives for deforestation under uncertainty.

Moreover, parameter estimation in stochastic systems presents significant challenges. Traditional approaches such as maximum likelihood estimation or Bayesian inference often rely on strong distributional assumptions and require high-resolution time series data — conditions that may be difficult to meet in forest monitoring scenarios. In recent years, machine learning, particularly deep learning, has emerged as a powerful tool for parameter inference in complex dynamical systems. For example, physics-informed neural networks have been applied to infer parameters in partial differential equation models from sparse or noisy data \cite{Raissi2019}. 


Our research objectives are twofold. First, we prove the existence of global positive solutions and investigate the impact of model parameters on the net expected gain from deforestation within the framework of our stochastic model. By identifying numerical thresholds for key control parameters, we aim to delineate conditions under which the net expected gain from deforestation changes sign. This information is crucial as it directly influences landowner decisions: a positive gain incentivizes deforestation for agricultural expansion, while a negative gain promotes forest cover growth.

Second, we develop a novel parameter estimation approach to address the critical challenge of parameter estimation for our stochastic model. 
Traditional methods, such as maximum likelihood estimation, often rely on strong distributional assumptions and require substantial real-world data. This can be problematic in applications like forest land-use change analysis where data collection is often limited. To overcome these challenges, we introduce a novel data-driven approach that integrates deep learning techniques. This approach leverages a hybrid dataset comprising limited real-world observations and a substantial amount of synthetic data generated from our stochastic model. 
Real-world data is obtained by observing the proportions of forest and agricultural land at multiple time points within a given period (e.g., monthly observations for a year). 
Synthetic data is generated by numerically solving the stochastic forest transition model. By effectively combining these datasets, our deep learning-based approach enables accurate estimation of all model parameters. This innovative approach enables us to understand long-term forest transition dynamics and the net expected gain from deforestation at any future times.

The paper is organized as follows. Section \ref{Model_desription} describes the stochastic forest transition model.  Section~\ref{sec:global} establishes the existence of global  positive solutions for the system and identifies an invariant set within the domain $\{(x, y)\in \mathbb R^2 \mid x, y>0, 0<x+y<1\}$. Section \ref{sec:numer} examines the influence of model parameters on the net expected gain of deforestation. In Section \ref{sec:estimation}, we address the  parameter estimation problem by generating a synthetic dataset through parameter sampling and numerical solutions of the system. This dataset is subsequently used to train various machine learning models. Finally, Section \ref{conclusion} provides the conclusion.

\section{Model description}\label{Model_desription}
In this section, we introduce our stochastic differential equation model to describe the dynamics of forest transition. Let $X_t > 0$, $Y_t > 0$, and $Z_t > 0$ denote the areas (in hectares) of forest, agricultural, and abandoned lands, respectively, owned by a landowner at time $t \geq 0$. The corresponding proportions of these land-use types are defined as:
$$x_t=\frac{X_t}{X_t+Y_t+Z_t}> 0, \quad y_t=\frac{Y_t}{X_t+Y_t+Z_t}>0, \quad z_t=\frac{Z_t}{X_t+Y_t+Z_t}>0.$$
By definition, these proportions satisfy the constraint:
$$x_t+y_t+z_t=1.$$
Hence, knowledge of the dynamics of $(x_t, y_t)$ is sufficient to determine $z_t$ through the identity $z_t = 1 - x_t - y_t$.

We assume that the pair $(x_t, y_t)$ evolves according to the following system of stochastic differential equations:
\begin{equation}\label{SFTM}
	\left\{ \,
	\begin{split}
		&dx_t = \lbrack p(x_t)(1-x_t-y_t) - r(x_t)x_t \rbrack dt + \sigma_1 x_t(1-x_t-y_t)dw_t, \\
		&dy_t = \lbrack r(x_t)x_t - \eta y_t\rbrack dt + \sigma_2 y_t(1-x_t-y_t)dw_t,
	\end{split}
	\right.
\end{equation}
with initial value 
$$(x_0, y_0) \in \{(x, y)\in \mathbb R^2 \mid x, y>0, 0<x+y<1\}.$$
Here, the process $\{w_t\}_{t\geq 0}$ is a Brownian motion defined on a filtered complete probability space $(\Omega, \mathcal{F}, {\mathcal F}_{t \geq 0}, \mathbb P)$. The stochastic term $dw_t$ represents environmental or socioeconomic randomness affecting land-use change, with multiplicative noise intensities $\sigma_1 x_t(1-x_t-y_t)$ and $\sigma_2 y_t(1-x_t-y_t)$ for the forest and agricultural land ratios, respectively. The constants $\sigma_1$ and $\sigma_2$ are positive and represent the strength of stochastic perturbations. 

Importantly, this formulation ensures that the noise intensity vanishes when the corresponding land-use category is absent. Specifically, the noise term in the first equation of \eqref{SFTM} becomes zero when $x_t = 0,$ and similarly, the noise in the second equation vanishes when $y_t = 0$.  Since the proportion of abandoned land is given by $z_t = 1 - x_t - y_t$, the total noise intensity affecting it is $(\sigma_1 x_t + \sigma_2 y_t)(1 - x_t - y_t)$, which vanishes when $z_t = 0$. This structure reflects the ecological reality that random fluctuations should not impact land-use types that are no longer present.

Let us now explain other parameters in the model \eqref{SFTM}. 
The function $p(x_t)= \mu + hx_t$ represents the forest recovery rate at time $t$. It is determined by the basic recovery rate $\mu \in (0,1)$ (when all forest land is deforested) and the coefficient of forest recovery $h > 0$. 
As $0 < p(x) < 1$, the parameters $\mu$ and $h$ satisfy $\mu + h < 1$. In addition, the parameter  $\eta \in (0,1)$ denotes the abandonment rate.

In the meantime, $r(x_t)$ denotes the deforestation rate when the extent of forest cover is  $x_t$. It is given~by:
\begin{align*}
	r(x_t) = \frac{1}{1+e^{-\beta G(x_t)}},
\end{align*}
where $\beta > 0$ controls the stochasticity in decision-making, and $G(x_t)$ is the net expected gain from deforestation when the extent of forest cover is  $x_t$. This net expected gain is defined as
$$G(x)=V_A(x)-V_F(x),$$
where $V_A$, $V_F$, and $V_E$ represent the expected discounted utilities of agricultural, forested, and abandoned land, respectively, when the extent of forest cover is $x$. These utility functions are interconnected as follows \cite{Satake}:
\begin{equation}
	\begin{cases}
		V_F=\frac{q(x)} {1- \gamma},\\
		V_A=\alpha + \gamma [(1-\eta)V_A+\eta V_E],\\
		V_E=\gamma \{[1-p(x)]V_E+p(x)V_F\}.
	\end{cases}
\end{equation}
Here, $q(x)$ represents the forest value for ecosystem services when the forest land proportion is $x$. The parameter  $\gamma \in (0,1)$ is the discount factor, and $\alpha > 0$ is the utility of agriculture.

By solving this system of equations, we obtain
\begin{equation*}
	\begin{aligned}
		G(x) &= \frac{\alpha [1-\gamma\{1 - p(x)\}] - [ 1 - \gamma \{ 1 - \eta - p(x)\}]q(x)}{1 - \gamma [2 - \eta - p(x)]  + \gamma^2 (1-\eta)[1 - p(x) ]} \\
		&=\frac{\alpha [1-\gamma\{1 - p(x)\}] - [ 1 - \gamma \{ 1 - \eta - p(x)\}]q(x)}{\{1- \gamma (1 - \eta)\} \{1 - \gamma (1 - p(x))\}}.
	\end{aligned}
\end{equation*} 
The function $G(x)$ is defined for $x > 0$ and crucially influences landowner decisions. A positive $G(x_t)$ at time $t$ incentivizes deforestation for agricultural expansion, while a negative $G(x_t)$ promotes forest cover increase using agricultural or abandoned land.

Following the deterministic case, we incorporate two alternative hypotheses regarding the forest value function  $q(x),$ each reflecting distinct economic and policy perspectives:
\newpage
\begin{itemize}
	\item [(FSH)] {\bf Forest Scarcity Hypothesis}: 
	
	Here, $q(x)$ represents the income from forest product sales (e.g., fuelwood, timber), modeled as  $q(x) = \delta + \lambda(1-x)$, where $ \delta \in (0, \alpha], \delta<1$ is the base return and $\lambda > 0$ captures the increasing value of forest products as forests become scarcer (i.e., as forest cover 
	$x$ declines). This reflects a market-based mechanism where deforestation increases the scarcity—and thus price—of forest goods. The (FSH) implies that policies may need to regulate extraction or create incentives for sustainable harvesting to prevent overexploitation driven by rising short-term profits.
	\item [(ESH)] {\bf Ecosystem Service Hypothesis}:
	
	In this case, $q(x) = \delta + \lambda x$ \, $(\delta \in (0,1), \lambda>0, 0 < \delta + \lambda \leq \alpha$)   reflects government payments or subsidies for ecosystem services, such as carbon sequestration or watershed protection. As forest cover $x$ increases, the value $q(x)$ rises, capturing the idea that intact forests provide greater environmental benefits. This hypothesis underpins conservation-oriented policies, such as payments for ecosystem services, where landowners are incentivized to maintain or restore forest cover for long-term ecological gains.
\end{itemize}
Both hypotheses influence land management decisions differently: (FSH) suggests a reactive valuation based on scarcity, potentially promoting short-term economic exploitation unless counterbalanced by regulation, while (ESH) emphasizes proactive valuation tied to conservation outcomes, aligning economic incentives with environmental stewardship.

To conclude this section, we summarize the constraints imposed on the ten parameters of our model~\eqref{SFTM}: 
\begin{equation*}
	\begin{cases}
		\sigma_1>0, \sigma_2>0, \mu>0, h>0, \beta>0, \alpha>0, \lambda>0,\\
		\mu+h<1, \\
		0<\eta, \gamma, \delta<1,\\
		0<\delta\leq \alpha \quad \text {in the case of  (FSH),}\\
		0<\delta+\lambda\leq  \alpha  \quad  \text{in the case of  (ESH). }\\
	\end{cases}
\end{equation*}

\section{Global solutions} \label{sec:global}
This section establishes the existence and uniqueness of solutions for our model \eqref{SFTM} under both the (FSH) and (ESH) hypotheses. We prove the existence of a unique global positive solution to \eqref{SFTM}. For foundational concepts in stochastic differential equations, refer to \cite{Friedman,Mao,Arnold}.
\begin{theorem} \label{theorem1}
	For any initial condition in  the triangle $\Delta \coloneqq\{(x, y)\in \mathbb R^2 \mid x, y>0, 0<x+y<1\}$, 
	there exists a unique global solution $(x_t,y_t)$ of \eqref{SFTM}. Furthermore,  $(x_t,y_t) \in \Delta $  a.s. for $0<t<\infty$.
\end{theorem}

\begin{proof}
	We only prove the theorem under the hypothesis  (FSH) (the proof for the (ESH) case is analogous). 
	
	Since the coefficients of the system \eqref{SFTM} are $C^\infty$-functions, they are locally Lipschitz continuous. Consequently, there is a unique local solution $(x_t,y_t)$ defined on an interval $[0,\tau)$, where $\tau $ is a stopping time. We know that \cite{Mao,ta2015sustainability}, if $\mathbb{P} \{ \tau <\infty\} >0,$ then $\tau$ is an explosion time on ${\tau <\infty}$, i.e.,  on $\{\tau <\infty\},$
	\begin{align*}
		\lim_{t\rightarrow \tau}(x_t + y_t) = 1, \quad \text{or}\quad  \lim_{t\rightarrow \tau}x_t = 0, \quad \text{or}\quad \lim_{t\rightarrow \tau}y_t = 0.
	\end{align*}
	To prove the theorem, it therefore suffices to show that $\tau = \infty$ a.s.

	Let $k_0$ be a positive integer such that $x_0, y_0 \geq \frac{1}{k_0}$ and $x_0+y_0\leq 1-\frac{1}{k_0}$. For each integer $k\geq k_0$, define a stopping time $\tau_k$ as $\tau_k=\inf \Delta_k$  with the convention $\inf\varnothing =\infty$, where 
	$$\Delta_k= \{t\mid 0\leq t < \tau\mid x_t \, \text{ or }\, y_t < \frac{1}{k}, \text{or}\ x_t + y_t > 1 - \frac{1}{k}\}.$$ 
	Since the sequence $\{\tau_k\}_{k=k_0}^\infty$ is nondecreasing, there exists $\tau_\infty= \lim_{k\rightarrow\infty}\tau_k$. Clearly, $\tau_\infty \leq \tau$ a.s. Hence, to prove that $\tau =\infty$ a.s., it suffices to show that $\tau_\infty =\infty$ a.s. 
	
	Assuming the contrary (i.e., $\mathbb P(\tau_\infty <\infty)>0$), there exist $T>0$ and $0<\epsilon <1$ such that 
	\begin{align*}
		\mathbb P(\{\tau_\infty <T \})>\epsilon.
	\end{align*}
	Then, we consider a positive function $V(x, y)$ defined on the triangle  $\Delta$ by
	\begin{align*}
		V(x, y)=-\log x-\log y-\log (1-x-y).
	\end{align*}
	We have for $(x,y)\in\Delta$,
	$$\frac{\partial V}{\partial x}=-\frac{1}{x} +\frac{1}{1-x-y}, \quad \frac{\partial V}{\partial y}=-\frac{1}{y} +\frac{1}{1-x-y}, $$
	and
	$$\frac{\partial^2 V}{\partial x^2}=\frac{1}{x^2} +\frac{1}{(1-x-y)^2}, \quad \frac{\partial^2 V}{\partial y^2}=\frac{1}{y^2} +\frac{1}{(1-x-y)^2}, 
	\quad \frac{\partial^2 V}{\partial x\partial y}=\frac{1}{(1-x-y)^2}.$$
	Applying the It\^o formula to $V$, we obtain that for $t\in [0,\tau),$
	\begin{align*}
		dV(x_t,y_t) =\, &\Big\{[p(x_t)(1-x_t-y_t) - r(x_t)x_t] \frac{\partial V(x_t, y_t)}{\partial x}+\frac{\sigma_1^2 x_t^2(1-x_t-y_t)^2}{2}\frac{\partial^2 V(x_t, y_t)}{\partial x^2} \\
		&+[r(x_t)x_t - \eta y_t]\frac{\partial V(x_t, y_t)}{\partial y}+\frac{\sigma_2^2 y_t^2(1-x_t-y_t)^2}{2}\frac{\partial^2 V(x_t, y_t)}{\partial y^2} \\
		&+\sigma_1\sigma_2 xy (1-x-y)^2\frac{\partial^2 V}{\partial x\partial y} \Big\} dt \\
		&+ \Big[\sigma_1 x_t(1-x_t-y_t) \frac{\partial V(x_t, y_t)}{\partial x}+\sigma_2 y_t(1-x_t-y_t) \frac{\partial V(x_t, y_t)}{\partial y}\Big]dw_t\\
		=\, &[V^{-}(x_t, y_t)+V^{+}(x_t, y_t)]dt +[\sigma_1(2x_t+y_t-1)+\sigma_2(x_t+2y_t-1)]dw_t,
	\end{align*}
	where
	\begin{align*}
		\begin{split}
			&V^-(x, y)= -\left\{\frac{1}{x}p(x)(1-x-y) + \frac{x}{y}r(x) + \frac{\eta y}{1-x-y} \right \}, \\
			&V^+(x,y) = r(x)+ p(x)+\eta +\frac{1}{2}(1-x-y)^2(\sigma_1^2+\sigma_2^2) + \frac{1}{2}\sigma_1^2x^2 + \frac{1}{2}\sigma_2^2y^2 +\sigma_1\sigma_2 xy. \\
		\end{split}
	\end{align*}
	Taking the expectation of the two sides of the equation yields
	\begin{align*}
		\mathbb{E}V(x_t,y_t) = V(x_0,y_0)+\mathbb{E} \int_0^t V^{-}(x_s,y_s)ds + \mathbb{E} \int_0^t V^+ (x_s,y_s)ds, \quad t\in [0,\tau).
	\end{align*}
	Thus,
	\begin{align*}
		\mathbb{E}V(x_{T\wedge\tau_k},y_{t\wedge\tau_k}) = V(x_0,y_0)+\mathbb{E} \int_0^{T\wedge\tau_k} V^{-}(x_s,y_s)ds + \mathbb{E} \int_0^{T\wedge\tau_k} V^+ (x_s,y_s)ds.
	\end{align*}
	It is easy to verify that $V^-(x,y)$ is non-positive and $V^+(x,y)$ is bounded by a constant $M$ on $\Delta$.
	Consequently,
	\begin{equation*}
		\begin{aligned}
			\mathbb EV(x_{T\wedge\tau_k}, y_{T\wedge\tau_k}) \leq V(x_0, y_0) + M\mathbb E(T\wedge\tau_k).
		\end{aligned}
	\end{equation*}
	
	Note that on the event $\{\tau_\infty <T \}$, $\tau_k <T$ for any $k\geq k_0.$
	In addition, evaluating $V(x_{\tau_k}, y_{\tau_k})$ at time $\tau_k$,  we observe that $(x_{\tau_k}, y_{\tau_k})$ lies on the boundary of  $\Delta_k$. Specifically, either $x_{\tau_k} = \frac{1}{k}$, $y_{\tau_k} = \frac{1}{k},$ or $x_{\tau_k} + y_{\tau_k} = 1-\frac{1}{k}$. Hence, 
	\begin{align*}
		\begin{split}
			V(x_{\tau_k}, y_{\tau_k}) &= -\log (x_{\tau_k}) - \log (y_{\tau_k}) - \log (1-x_{\tau_k} - y_{\tau_k}) \\
			&\geq -\log (\frac{1}{k}) = \log k, \quad k\geq k_0.
		\end{split}
	\end{align*}
	Combining these results, we obtain for $k\geq k_0,$
	\begin{equation*}
		\begin{aligned}
			\infty &> V(x_0, y_0) + MT \\
			&> V(x_0, y_0) + M\mathbb E(T\wedge\tau_k) \\
			&\geq \mathbb EV(x_{T\wedge\tau_k}, y_{T\wedge\tau_k}) \\
			&\geq \mathbb E \lbrack \mathbbm{1}_{\{\tau_\infty<T\}} V(x_{T\wedge \tau_k}, y_{T\wedge \tau_k}) \rbrack \\
			&= \mathbb E \lbrack \mathbbm{1}_{\{\tau_\infty<T\}} V(x_{\tau_k}, y_{\tau_k}) \rbrack \\
			&>\epsilon \log k.
		\end{aligned}
	\end{equation*}
	This leads to a contradiction as $k$ approaches infinity: $\infty > V(x_0, y_0) + MT\geq \lim_{k\to \infty} \epsilon \log k =\infty$. Therefore, $\tau_\infty = \tau = \infty$ a.s.,
	implying that $x_t, y_t >0$ and $x_t + y_t <1$ a.s. for  $0\leq t <\infty.$
\end{proof}

\begin{remark}
	We recall that two hypotheses (FSH) and (ESH) define the linear function $q(x)$.
	Although we proved Theorem \ref{theorem1} under (FSH) and (ESH), Theorem \ref{theorem1} holds for any function $q(x)$ on $\Delta$ valued in $\mathbb{R}$. Indeed, we used the fact that the function $r(x)$, which is defined by using $q(x)$, is positive valued and bounded.
\end{remark}

\section{Impact of parameters on deforestation decisions} \label{sec:numer}
This section delves into the influence of model parameters on the net expected gain of deforestation. After illustrating sample paths to demonstrate Theorem \ref{theorem1}, we identify numerical thresholds for the parameters in model \eqref{SFTM} that determine the sign of expectation of the net expected gain. Crucially, this gain significantly impacts landowner decisions. A positive net gain in expectation incentivizes deforestation for agricultural expansion, while a negative net gain promotes forest conservation.

First, let us illustrate the trajectories of the solution. Figures \ref{figure_solution_FSH} and \ref{figure_solution_(ESH)} present sample paths of the solution to \eqref{SFTM} under the hypotheses (FSH) and (ESH), respectively. The left panels depict the time series of the forest land ratio ($x_t$) and agricultural land ratio ($y_t$), while the right panels illustrate the relationship between $x_t$ and $y_t$. As shown, the solution trajectories remain within the triangle $\{(x, y)\in \mathbb R^2 \mid x, y>0, 0<x+y<1\}$, aligning with the findings of Theorem \ref{theorem1}. These simulations are generated using the Euler-Maruyama method \cite{Kloeden2003} with an initial condition $(x_0, y_0)=(0.2,0.3)$ and parameter values: $\mu = 0.2$, $h = 0.3$, $\eta = 0.7, \beta = 2, \delta = 0.7,\lambda = 1, \gamma = 0.5, \alpha = 2,$ and $ \sigma_1 = \sigma_2 = 1$. To ensure numerical convergence, a small time step of $dt = \frac{1}{999}$ is employed.

\begin{figure}[H]
	\begin{minipage}{0.6\hsize}
		\includegraphics[width=0.8\linewidth]{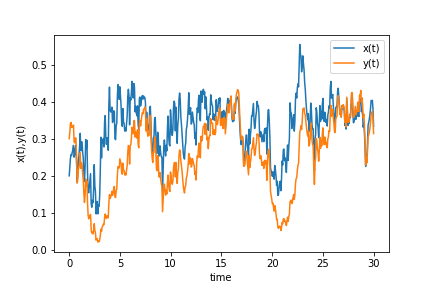}
	\end{minipage}
	\begin{minipage}{0.3\hsize}
		\includegraphics[keepaspectratio, scale=0.4]{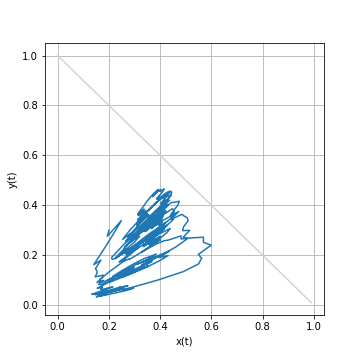}
	\end{minipage}
	\caption{A sample of the ratios of forest land $x_t$ and agricultural land $y_t$ (left) and of $(x_t,y_t)$ (right) with $t\in [0, 30]$ under (FSH).}  
	\label{figure_solution_FSH}
\end{figure}

\begin{figure}[H]
	\begin{minipage}{0.6\hsize}
		\includegraphics[width=0.8\linewidth]{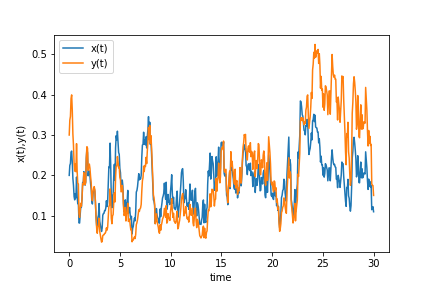}
	\end{minipage}
	\begin{minipage}{0.3\hsize}
		\includegraphics[keepaspectratio, scale=0.4]{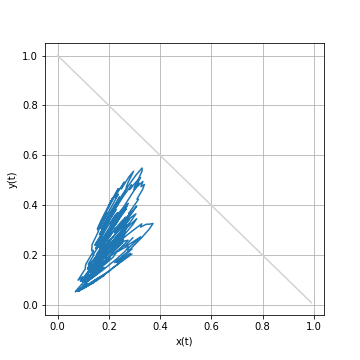}
	\end{minipage}
	\caption{A sample of the ratios of forest land $x_t$ and agricultural land $y_t$ (left) and of $(x_t,y_t)$ (right) with $t\in [0, 30]$ under (ESH).}
	\label{figure_solution_(ESH)}
\end{figure}

In addition, Figure \ref{fig:enter-label} displays the distribution of $(x_T, y_T)$ at time $T = 30$ under both the (FSH) and (ESH) hypotheses. This distribution is generated using 500 sample paths of  solutions whose initial values are taken randomly. As evident from the figure, the support of the distribution is confined to the triangle $\{(x, y)\in \mathbb R^2 \mid x, y>0, 0<x+y<1\}$, corroborating the findings of Theorem \ref{theorem1}.

Second, to determine the influence of model parameters on deforestation decisions, we analyze the sign of expectation of the net expected gain, $\mathbb E G(x_T)$, at time $T = 30$. 
Recall that 
\begin{align*}
	G(x) = \frac{\alpha [1-\gamma\{1 - p(x)\}] - [ 1 - \gamma \{ 1 - \eta - p(x)\}]q(x)}{1 - \gamma [2 - \eta - p(x)]  + \gamma^2 (1-\eta)[1 - p(x) ]}.
\end{align*}
A positive $\mathbb EG(x_T)$ indicates a preference for deforestation to expand agricultural land, while a negative value promotes forest conservation. For these analyses, we use the initial condition $(x_0, y_0) = (0.2, 0.3)$.
\begin{figure}[H]
	\begin{minipage}{0.5\hsize}
		\centering
		\includegraphics[keepaspectratio, scale=0.4]{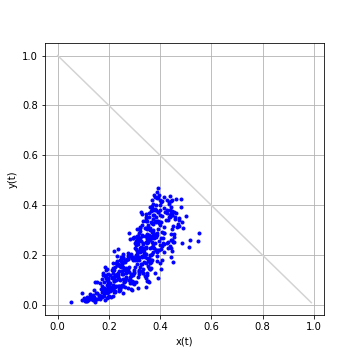}
	\end{minipage}
	\begin{minipage}{0.5\hsize}
		\centering
		\includegraphics[keepaspectratio, scale=0.4]{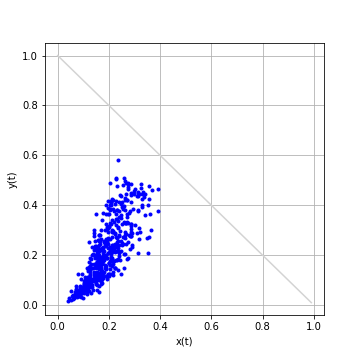}
	\end{minipage}
	\caption{A distribution of $(x_T, y_T)$ at $T = 30$ under (FSH) (left) and (ESH) (right).}
	\label{fig:enter-label}
\end{figure}
Under the (FSH) hypothesis, Figure \ref{G_change_5parameters_FSH} illustrates the relationship between $\mathbb EG(x_T)$ and each of the five parameters ($\eta$, $\delta$, $\lambda$, $\gamma$, and $\alpha$ in this order) while holding the others constant. For each parameter value, $\mathbb EG(x_T)$ is calculated as the average of 100 sample paths at time $T$. The results indicate that, except for $\gamma$, the sign of $\mathbb EG(x_T)$ transitions from positive to negative once as $\eta$, $\delta$, and $\gamma$ increase from 0 to 1, and $\lambda$ from 0 to 2. Conversely, $\mathbb EG(x_T)$ changes from negative to positive as $\alpha$ increases from 0 to~3. Table \ref{tb:sign_G_FSH} summarizes the parameter values at which $\mathbb EG(x_T) = 0$.

\begin{figure}[H]
	\centering
	\includegraphics[keepaspectratio, scale=0.4]{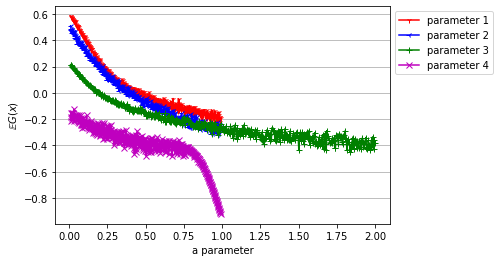}
	\includegraphics[keepaspectratio, scale=0.4]{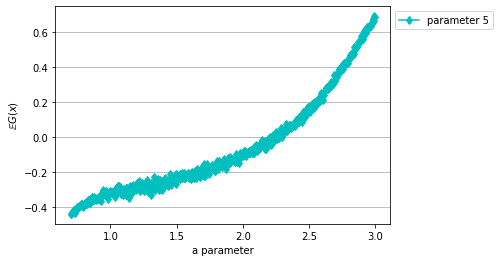}
	\caption{Sensitivity of $\mathbb EG(x_T)$  to individual parameters under (FSH). Each curve represents the variation of $\mathbb EG(x_T)$
		with respect to a specific parameter $i $ (where $i=1,\dots,5).$ The parameters correspond to $(\eta, \delta,\lambda, \gamma, \alpha)$.  Other parameters are fixed at $\mu = 0.2, h =  0.3, \eta = 0.7, \sigma_1 = \sigma_2 = 1, \beta = 2, \delta = 0.7,\lambda = 1, \gamma = 0.5.$}
	\label{G_change_5parameters_FSH}
\end{figure}

\begin{table}[htbp]
	\centering	\caption{The values of parameters at which the sign of $\mathbb EG(x_T)$ changes under (FSH).}
	\begin{tabular}{lllll}
		\hline
		$\eta$ &  $\delta$ & $\lambda$ & $\alpha$ \\ \hline
		0.4381 & 0.3713& 0.1934&2.2026\\ \hline
	\end{tabular}

	\label{tb:sign_G_FSH}
\end{table}

The effect of noise is also considered. Figure \ref{G_x_noise_FSH} illustrates the relationship between $\mathbb EG(x_T)$ and the noise magnitude ($\sigma_1 = \sigma_2$) under the (FSH) hypothesis, while keeping other parameters constant. The figure reveals that $\mathbb EG(x_T)$ remains negative for noise levels between 0 and 1. However, multiple sign changes occur as noise increases from 1 to 3, with approximate thresholds at $\sigma_1 = \sigma_2 = 1.88$ and $2.10$.

\begin{figure}[H]
	\centering
	\includegraphics[keepaspectratio, scale=0.4]{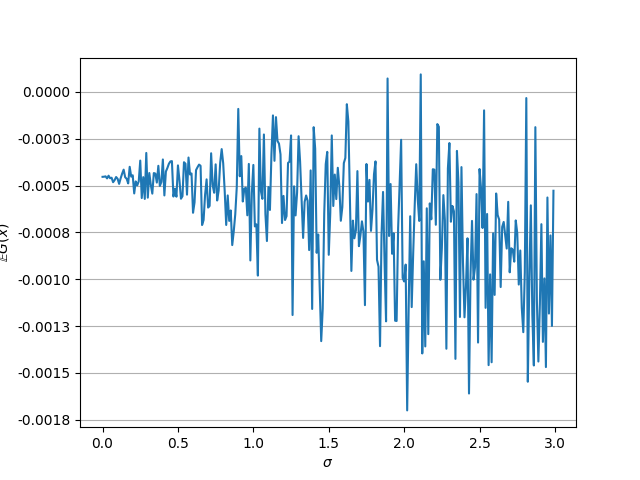}
	\caption{Graph of $\mathbb EG(x_T)$ along  $\sigma = \sigma_1 = \sigma_2$ under (FSH). Other parameters are fixed at $\mu = 0.2, h = 0.3,\eta = 0.7, \beta = 2, \delta = 0.7,\lambda = 1, \gamma = 0.5, $ and $\alpha = 2.$}
	\label{G_x_noise_FSH}
\end{figure}
\begin{figure}[H]
	\centering
	\includegraphics[keepaspectratio, scale=0.4]{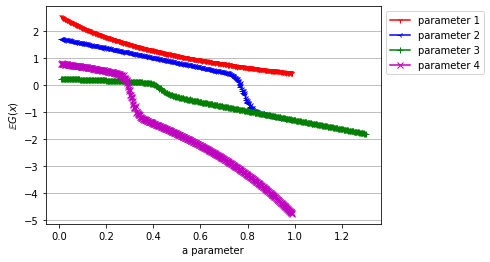}
	\includegraphics[keepaspectratio, scale=0.4]{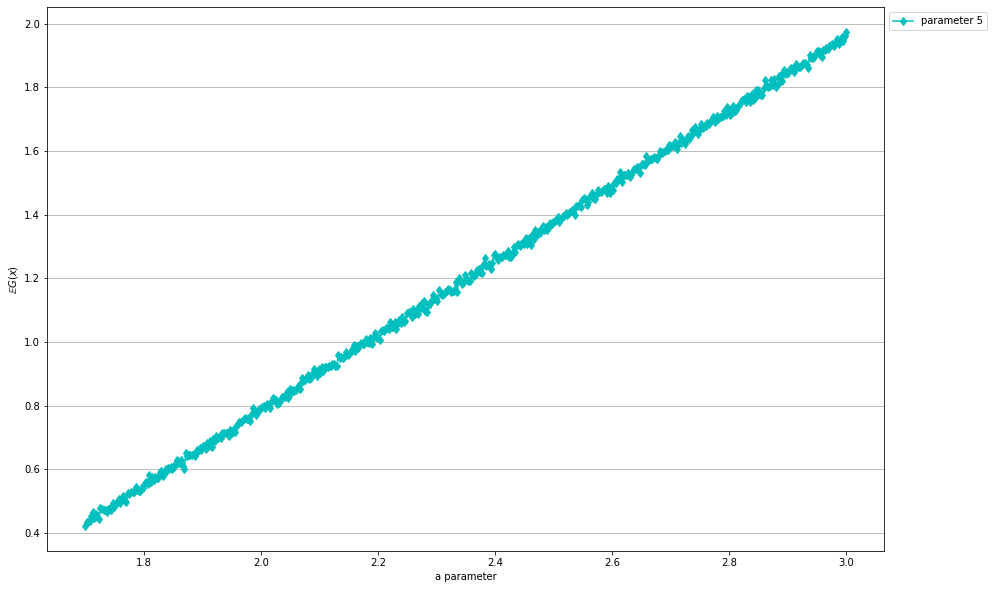}
	\caption{Sensitivity of $\mathbb EG(x_T)$  to individual parameters under (ESH). Each curve represents the variation of $\mathbb EG(x_T)$
		with respect to a specific parameter $i $ (where $i=1,\dots,5).$  The parameters correspond to $(\eta, \delta,\lambda, \gamma, \alpha)$.  Other parameters are fixed at  $\mu = 0.2, h =  0.3, \eta = 0.7, \sigma_1 = \sigma_2 = 1, \beta = 2, \delta = 0.7,\lambda = 1, $ and $\gamma = 0.5.$}
	\label{G_change_5parameters_ESH}
\end{figure}
Similar analyses are conducted under the (ESH) hypothesis. Figure \ref{G_change_5parameters_ESH} shows the relationship between $\mathbb EG(x_T)$ and each of the five parameters while holding others constant. Unlike the (FSH) case, the sign of $\mathbb EG(x_T)$ transitions from positive to negative only for $\delta$, $\lambda$, and $\gamma$, as these parameters increase. Table~\ref{tb:sign_G_(ESH)} summarizes the corresponding threshold values where $\mathbb EG(x_T) = 0$. 

Notably, noise levels between 0 and 3 do not alter the sign of $\mathbb EG(x_T)$ under the (ESH) hypothesis, although $\mathbb EG(x_T)$ exhibits significant fluctuations at higher noise levels (Figure \ref{G_x_noise_ESH}).

\begin{table}[htbp]
	\centering 	\caption{The values of parameters at which the sign of $\mathbb EG(x_T)$ changes under (ESH).}
	\begin{tabular}{lll}
		\hline
		$\delta$ & $\lambda$ & $\gamma$   \\ \hline
		0.7661& 0.4081 & 0.2928 \\ \hline
	\end{tabular}

	\label{tb:sign_G_(ESH)}
\end{table}

\begin{figure}[htbp]
	\centering
	\includegraphics[keepaspectratio, scale=0.6]{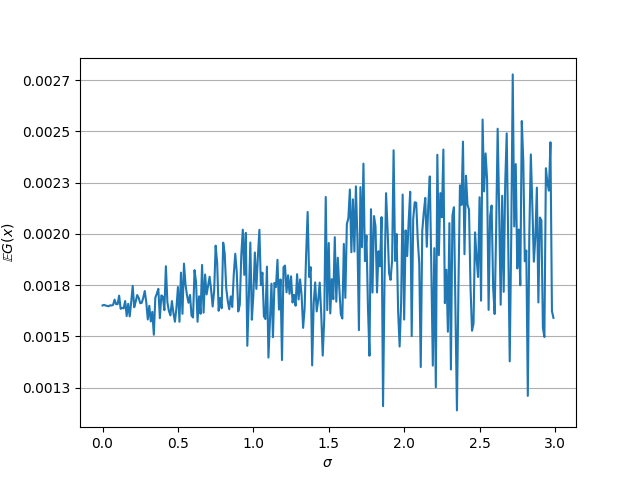}
	\caption{Graph of $\mathbb EG(x_T)$ along $\sigma_1$ and $\sigma_2$ when $\sigma_1 = \sigma_2$ under (ESH). Other parameters are as follows: $\mu = 0.2, h = 0.3,\eta = 0.7, \beta = 2, \delta = 0.7,\lambda = 1, \gamma = 0.5, $ and $\alpha = 2.$}
	\label{G_x_noise_ESH}
\end{figure}

\section{Parameter estimation via deep learning}  \label{sec:estimation}
This section introduces a novel approach to estimating parameters in the stochastic forest transition model  \eqref{SFTM} using machine learning techniques. Traditionally, parameter estimation for stochastic models relies on methods like maximum likelihood, which requires specific distributional assumptions and a sufficient amount of data per parameter. However, in many real-world scenarios, including forest land-use change analysis, obtaining such data can be challenging.

To overcome these limitations, we propose a data-driven approach that leverages the power of deep learning and machine learning. We utilize observed time series data of forest and agricultural land cover, denoted as $(x_t, y_t)$ for time period $t = 1, ..., T$, to estimate the eight model parameters ($\mu, h, \eta, \beta, \delta, \lambda, \gamma, \alpha$) and noise intensities ($\sigma_1, \sigma_2$).

Our methodology involves the following steps:
\begin{itemize}
	\item [(DG)] {\bf Data Generation}: A large synthetic dataset of model parameters ($P, \sigma$) is generated through uniform sampling within plausible ranges. This ensures that the dataset of parameters for training is sufficiently rich to encompass a wide range of possible parameter values.
	\item [(MS)] {\bf Model Simulation}: For each parameter set in the synthetic dataset, the system \eqref{SFTM} coupled with initial value $(x_1,y_1)$ is numerically solved to generate simulated time series data $(A, F)$.
	\item [(MT)] {\bf Model Training}: Deep learning (recurrent neural network (RNN), long short-term memory (LSTM), and shallow convolutional neural network (SCNN)), and machine learning (random forest) models are trained to predict the original parameters ($P, \sigma$) based on the simulated time series data $(A, F)$.
\end{itemize}

By training these models on a vast synthetic dataset, we aim to capture complex relationships between model parameters and the observed land-use dynamics. The optimal model will be selected based on its performance in predicting the original parameters.

Before delving into the specifics, it is essential to provide a brief overview of the deep learning and machine learning models employed in this research \cite{d2l}.

{\bf RNNs:} RNNs are a class of neural networks designed to process sequential data. Unlike traditional feedforward neural networks, RNNs incorporate a recurrent connection, allowing them to maintain a form of "memory" about previously processed elements. This enables RNNs to effectively capture temporal dependencies inherent in time-series data such as forest and agricultural metrics.


{\bf LSTMs:} LSTMs are a specialized type of RNN that address the vanishing gradient problem often encountered in traditional RNNs. LSTMs introduce memory cells and gate mechanisms to regulate the flow of information, enabling them to learn long-term dependencies more effectively. This characteristic makes LSTMs particularly well-suited for capturing the seasonal and trend patterns prevalent in forest and agricultural data.

Given the cyclical nature and temporal dependencies often observed in forest and agricultural datasets, RNNs and LSTMs are well-positioned to extract meaningful patterns and inform accurate predictions.



{\bf SCNNs:} To capture the intricate relationships between forest and agricultural data, we employed a SCNN architecture. Unlike traditional CNNs, which excel in processing grid-like data (e.g., images), our data consisted of time series for forest and agriculture metrics. To adapt CNNs to this structure, we combined the two time series into a single input matrix, forming two channels. Due to the relatively short length of our time series, employing a deep CNN architecture was impractical. Consequently, we opted for a shallow CNN to effectively extract relevant features while maintaining computational efficiency.

Preliminary experiments involving data transformation into trend figures yielded inferior results compared to using raw data. Therefore, the presented model exclusively utilizes raw data for training and evaluation.

{\bf Random Forest:} To provide a comparative analysis with deep learning methods, we incorporated random forest, a well-established ensemble machine learning technique, into our study. Known for its robust performance on classification tasks, random forest offers several advantages. Its ensemble nature, combining multiple decision trees, enhances predictive accuracy while mitigating overfitting. Additionally, random forest generally requires fewer hyperparameter adjustments compared to deep neural networks, simplifying the modeling process. Given its ability to handle large datasets and its interpretability through feature importance analysis, random forest serves as a suitable baseline for evaluating the performance of our proposed deep learning models.

We now elaborate on steps (DG), (MS), and (MT). For step (DG), we uniformly sample $n_1 = 20,000$ sets of model parameters ($\mu, h, \eta, \beta, \delta, \lambda, \gamma, \alpha$) and noise intensities ($\sigma_1, \sigma_2$), forming matrices  of size $n_1 \times 8$ and $n_1 \times 2$:
$$
\begin{bmatrix}
	p_{11} & p_{12} & \dots  & p_{18}\\
	p_{21} & p_{22} & \dots  & p_{28}\\
	\vdots & \vdots & \vdots & \vdots \\
	p_{n_11}& p_{n_12} & \dots  & p_{n_18}
\end{bmatrix}
\text{and }
\begin{bmatrix}
	\sigma_{11} & \sigma_{12} \\
	\sigma_{21} & \sigma_{22} \\
	\vdots & \vdots  \\
	\sigma_{n_11}& \sigma_{n_12} 
\end{bmatrix}
$$
Notice that we assume a relatively small impact of noise on the system, i.e., $\sigma_1, \sigma_2 \in (0,0.1)$. 

Due to the stochastic nature of \eqref{SFTM}, multiple samples $(x_t, y_t)$ can be generated for a given parameter set. To increase the variability of our dataset and enhance model learning, we replicate each parameter set $n_2 = 25$ times. This results in a dataset of size $n = n_1 \times n_2=500,000$ containing parameter sets and corresponding noise intensities, represented by matrices $P$ (size $n \times 8$) and $\sigma$ (size $n \times 2$), respectively:
$$
P=
\begin{bmatrix}
	p_{11} & p_{12} & \dots  & p_{18}\\
	p_{11} & p_{12} & \dots  & p_{18}\\
	\vdots & \vdots & \vdots & \vdots \\
	p_{11} & p_{12} & \dots  & p_{18}\\
	p_{21} & p_{22} & \dots  & p_{28}\\
	p_{21} & p_{22} & \dots  & p_{28}\\
	\vdots & \vdots & \vdots & \vdots \\
	p_{21} & p_{22} & \dots  & p_{28}\\
	\vdots & \vdots & \vdots & \vdots \\
	p_{n_11}& p_{n_12} & \dots  & p_{n_18}\\
	\vdots & \vdots & \vdots & \vdots \\
	p_{n_11}& p_{n_12} & \dots  & p_{n_18}
\end{bmatrix}
\text{and }
\sigma=
\begin{bmatrix}
	\sigma_{11} & \sigma_{12} \\
	\sigma_{11} & \sigma_{12} \\
	\vdots & \vdots  \\
	\sigma_{11} & \sigma_{12} \\
	\sigma_{21} & \sigma_{22} \\
	\sigma_{21} & \sigma_{22} \\
	\vdots & \vdots   \\
	\sigma_{21} & \sigma_{22} \\
	\vdots & \vdots  \\
	\sigma_{n_11}& \sigma_{n_12} \\
	\vdots & \vdots  \\
	\sigma_{n_11}& \sigma_{n_12} 
\end{bmatrix}
$$

For the Model Simulation step (MS), we numerically solve the system \eqref{SFTM} of stochastic differential equations  using the specified initial conditions $(x_1, y_1)$ for each parameter set in $(P, \sigma)$. This process generates a time series of simulated agriculture ($x_t$) and forest ($y_t$) land cover for each parameter set, spanning the time period $t=1, \dots, T$. The resulting simulated data are organized into matrices $A$ and $F$, respectively, where each row represents a time series $x_t$ and $y_t$:
$$
A=
\begin{bmatrix}
	x_{11} & x_{12} & \dots  & x_{1T}\\
	x_{21} & x_{22} & \dots  & x_{2T}\\
	\vdots & \vdots & \vdots & \vdots \\
	x_{n1}& x_{n2} & \dots  & x_{nT}
\end{bmatrix}
\text{ and }
F=
\begin{bmatrix}
	y_{11} & y_{12} & \dots  & y_{1T}\\
	y_{21} & y_{22} & \dots  & y_{2T}\\
	\vdots & \vdots & \vdots & \vdots \\
	y_{n1}& y_{n2} & \dots  & y_{nT}
\end{bmatrix}.
$$

We now use the data sets $(A, F,P,\sigma)$ of features and labels for the Model Training step (MT).
To train and evaluate our deep learning model, we randomly divide the combined dataset of features and labels into training and testing sets using an 80:20 split. This stratified random sampling ensures that the distribution of labels (classes) are maintained in both sets. The resulting training and testing sets are denoted as $(A^1, F^1, P^1, \sigma^1)$ and $(A^2, F^2, P^2, \sigma^2)$, respectively. 

In addition, to ensure consistent feature scales and improve model performance, we standardize the training and testing sets using max-min scaling. This normalization technique rescales features to a specific range (typically -1 to 0) by subtracting the maximal value and dividing by the range. 
This transformation helps prevent features with larger magnitudes from dominating the learning process:
$$
\tilde{x}_{ij} = 
\begin{cases}
	\frac{x_{ij} - \max(A^1) }{\max(A^1) -\min(A^1)}, 
	\text{ if } \max(A^1) \ne  \min(A^1),\\
	0 \text{ otherwise,}
\end{cases} 
$$    
and 
$$
\tilde{y}_{ij} = 
\begin{cases}
	\frac{y_{ij} - \max(F^1) }{\max(F^1) -\min(F^1)}, 
	\text{ if } \max(F^1) \ne  \min(F^1),\\
	0 \text{ otherwise.}
\end{cases} 
$$
Furthermore, in all training models, we employ a mean squared error (MSE) loss function to quantify the difference between predicted and actual values. The MSE is defined as follows:
\begin{align*}
	MSE(\hat{z}, z)= \frac{1}{p}\sum_{i=1}^p\|\hat{z}_i - z_i\|^2,
\end{align*}
where $\hat{z}_i$ represents the predicted value for the $i$-th data point, $z_i$ represents the corresponding true value, and $p$ is the total number of data points.
By minimizing the MSE during training, our models aim to learn a mapping between the input features $(A, F)$ and the target variable $(P, \sigma)$ that would minimize the squared error between predictions and true values.

\subsection {Training the RNN, LSTM, SCNN, and random forest under hypothesis (FSH)} \label{subsec4.1}
\begin{figure}[H]
	\centering
	\includegraphics[keepaspectratio, scale=0.4]{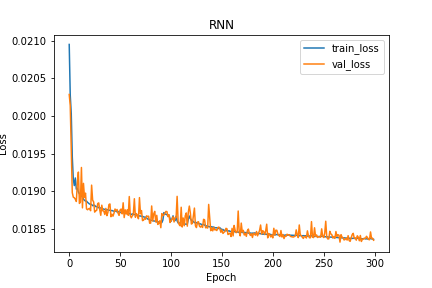}
	\includegraphics[keepaspectratio, scale=0.4]{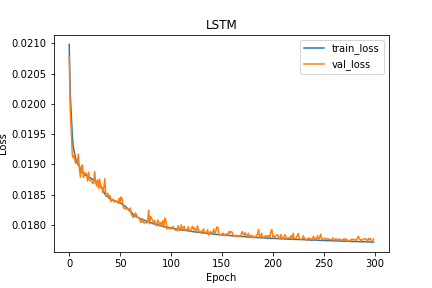}
	\includegraphics[keepaspectratio, scale=0.4]{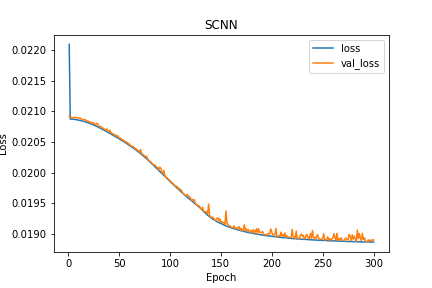}
	\caption{Performance of the RNN, LSTM, and SCNN estimating 8 parameters ($\mu, h, \eta, \beta, \delta, \lambda, \gamma, \alpha$) under (FSH).}
	\label{RNN_FSH}
\end{figure}
This subsection presents the results of training the RNN, LSTM, SCNN, and random forest under hypothesis (FSH). Given the number and types of target variables, we conduct separate training sessions for each of the eight target parameters ($\mu, h, \eta, \beta, \delta, \lambda, \gamma, \alpha$) and for the noise intensities $(\sigma_1,\sigma_2)$. All deep learning models (RNN, LSTM, and SCNN) are trained using 300 epochs with a batch size of 64. Meanwhile, for random forest, we use 1000 decision trees. 

First, we present results for predicting 8 parameters ($\mu, h, \eta, \beta, \delta, \lambda, \gamma, \alpha$).
Figure \ref{RNN_FSH} illustrates the performance of the RNN, LSTM, and SCNN, across the training epochs. As shown in the figure, the loss functions initially exhibit significant fluctuations but gradually converge to more stable values as the training progresses. 

Meanwhile, Table \ref{loss_FSH} presents the loss values for the RNN, LSTM, SCNN, and random forest models on the test set. Together with Figure \ref{RNN_FSH}, it indicates that the SCNN outperforms the other models in estimating the eight parameters under hypothesis (FSH).

\begin{table}[htbp]
	\centering 	\caption{Test loss of the RNN, LSTM, SCNN, and random forest when estimating 8 parameters ($\mu, h, \eta, \beta, \delta, \lambda, \gamma, \alpha$) under (FSH).}
	\begin{tabular}{ll}
		\hline
		Model & Test loss \\
		\hline
		RNN & 2.670648 \\ 
		LSTM &1.360166 \\ 
		SCNN &0.018867 \\  Random Forest & 0.029633
		\\ \hline
	\end{tabular}

	\label{loss_FSH}
\end{table}
\begin{figure}[H]
	\centering
	\includegraphics[keepaspectratio, scale=0.4]{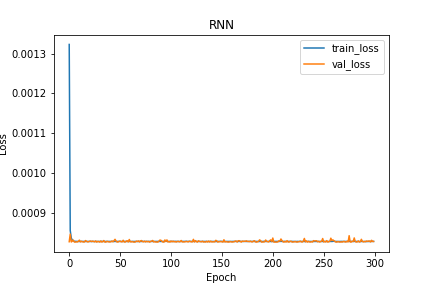}
	\includegraphics[keepaspectratio, scale=0.4]{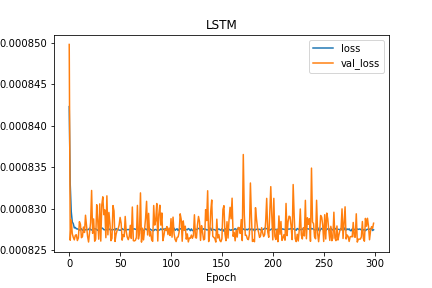}
	\includegraphics[keepaspectratio, scale=0.4]{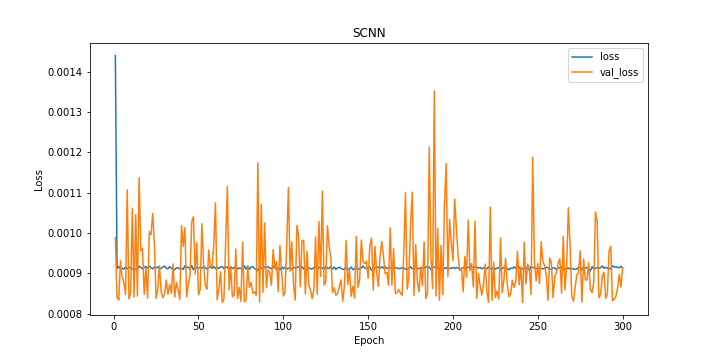}
	\caption{Performance of the RNN, LSTM, and SCNN estimating 2 parameters ($\sigma_1,  \sigma_2$) under (FSH).}
	\label{RNN_noise_FSH}
\end{figure}
Second, we present similar results for predicting two noise intensities   ($\sigma_1,  \sigma_2$) in Figure \ref{RNN_noise_FSH} and Table \ref{loss_noise_FSH}. Figure \ref{RNN_noise_FSH} shows that the training and validation losses exhibit minimal fluctuations after an initial few epochs. Meanwhile, Table \ref{loss_noise_FSH} demonstrates that the RNN, LSTM, and SCNN exhibit comparable performance in estimating noise intensities, surpassing random forest. Notably, the loss values for noise intensity estimation are generally lower than those observed for estimating the eight parameters, even when using the same models. This can be attributed to the assumption of a relatively small noise impact on the system, leading to a higher coverage level of the training set for noise parameters. While increasing the dataset size for parameter estimation could potentially reduce loss values, it would also incur significantly higher computational costs.

\begin{table}[H]
	\centering 	\caption{Test loss of the RNN, LSTM, SCNN, and random forest when estimating 2 parameters ($\sigma_1,  \sigma_2$) under (FSH).}
	\begin{tabular}{ll}
		\hline
		Model & Test loss \\
		\hline
		RNN & 0.000827 \\ 
		LSTM &0.000829 \\ 
		SCNN &0.000912 \\  Random Forest & 0.002403
		\\ \hline
	\end{tabular}

	\label{loss_noise_FSH}
\end{table}

\subsection {Training the RNN, LSTM, SCNN, and random forest under hypothesis (ESH)}
\label{subsec4.2}
This subsection replicates the experiments conducted in Subsection \ref{subsec4.1}, but now under the context of hypothesis (ESH). We maintain the same experimental setup regarding the number of epochs, batch size, and decision trees as used previously.

For the task of predicting eight parameters  ($\mu, h, \eta, \beta, \delta, \lambda, \gamma, \alpha$), 
Figure \ref{RNN_ESH} depicts the performance of the RNN, LSTM, and SCNN,  across the training epochs. In the figure, the loss function exhibits significant fluctuations during the initial epochs but gradually stabilizes as the training progresses.

\begin{figure}[H]
	\centering
	\includegraphics[keepaspectratio, scale=0.4]{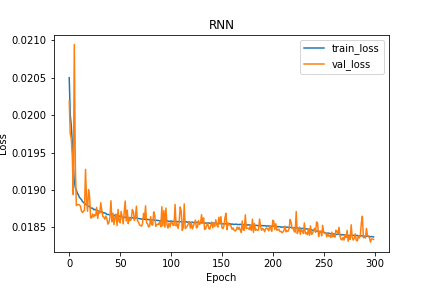}
	\includegraphics[keepaspectratio, scale=0.4]{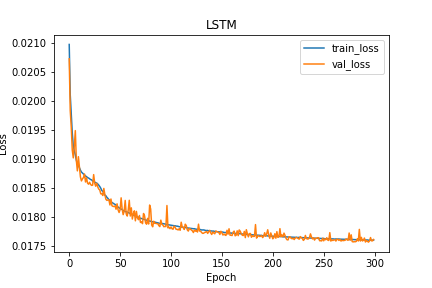}
	\includegraphics[keepaspectratio, scale=0.4]{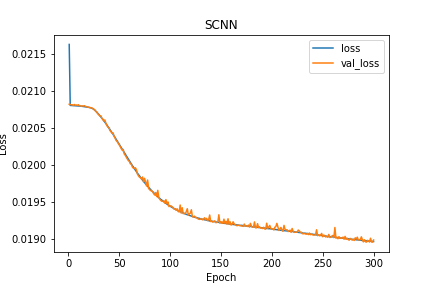}
	\caption{Performance of the RNN, LSTM, and SCNN estimating 8 parameters ($\mu, h, \eta, \beta, \delta, \lambda, \gamma, \alpha$) under (ESH).}
	\label{RNN_ESH}
\end{figure}

Table \ref{loss_ESH} presents the loss values of the RNN, LSTM, SCNN, and random forest models evaluated on the test set. Consistent with the findings in Subsection \ref{subsec4.1}, the SCNN consistently outperforms the other models in estimating all eight parameters.

\begin{table}[htbp]
	\centering 	\caption{Test loss of the RNN, LSTM, SCNN, and random forest when estimating 8 parameters ($\mu, h, \eta, \beta, \delta, \lambda, \gamma, \alpha$) under (ESH).}
	\begin{tabular}{ll}
		\hline
		Model & Test loss \\
		\hline
		RNN & 2.770403\\
		LSTM &1.292252 \\ 
		SCNN &0.018999 \\  Random Forest & 0.026122
		\\ \hline
	\end{tabular}

	\label{loss_ESH}
\end{table}

For predicting noise intensities  ($\sigma_1, \sigma_2$), Figure \ref{RNN_noise_ESH} illustrates the training dynamics of  the RNN, LSTM, and SCNN,  across the training epochs. The loss functions exhibit minimal fluctuations during training. Table \ref{loss_noise_ESH} presents the corresponding loss values on the test set for the RNN, LSTM, SCNN, and random forest. Similar to the findings under hypothesis (FSH), the RNN, LSTM, and SCNN demonstrate comparable performance in estimating noise intensities, outperforming random forest. Notably, the loss values for noise intensity estimation are generally lower than those observed for estimating the eight parameters, even when using the same models. 

\begin{figure}[H]
	\centering
	\includegraphics[keepaspectratio, scale=0.4]{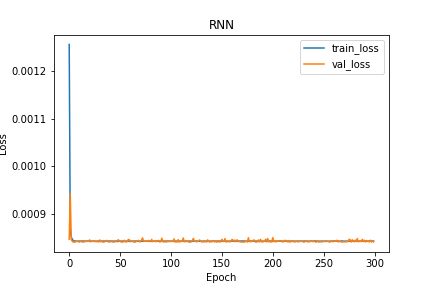}
	\includegraphics[keepaspectratio, scale=0.4]{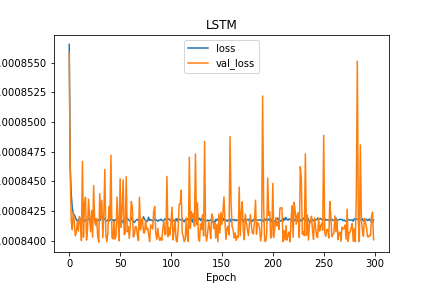}
	\includegraphics[keepaspectratio, scale=0.4]{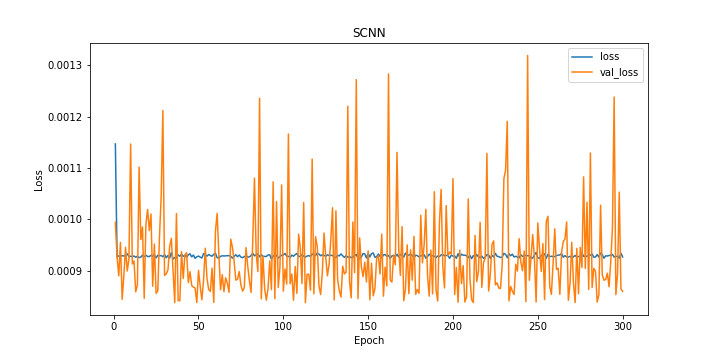}
	\caption{Performance of the RNN, LSTM, and SCNN estimating noise intensities ($\sigma_1,  \sigma_2$) under (ESH).}
	\label{RNN_noise_ESH}
\end{figure}

Based on the aforementioned findings, we conclude that the SCNN emerges as the most effective deep learning model for our task. The superior performance of the SCNN architecture observed in both eight-parameter and noise intensity estimation, under both (FSH) and (ESH), can be attributed to its ability to preserve the spatial structure of agriculture and forest data. Unlike the RNN and LSTM, which flatten the data into vectors, the SCNN maintains the spatial relationships between agriculture and forest components by attaching them to two channels within a tensor. This enables the application of 2D convolution, leading to improved feature extraction and overall model performance.
\begin{table}[H]
	\centering 	\caption{Test loss of  the RNN, LSTM, SCNN, and random forest when estimating  ($\sigma_1,  \sigma_2$) under (ESH).}
	\begin{tabular}{ll}
		\hline
		Model & Test loss \\
		\hline
		RNN & 0.000839 \\
		LSTM &0.000837 \\
		SCNN &0.000860 \\ Random Forest & 0.002260
		\\ \hline
	\end{tabular}
	
	\label{loss_noise_ESH}
\end{table}
\section{Conclusions} \label{conclusion}
This paper introduces a stochastic differential equation model to capture the dynamics of forest transition, extending existing deterministic discrete models. Through theoretical analysis, we establish the existence and uniqueness of global positive solutions within a biologically meaningful domain. Numerical simulations further illustrate the system's behavior and reveal how key parameters influence land-use decisions, particularly the trade-offs between deforestation and forest regeneration.

A central contribution of this work is the development of a novel deep learning-based parameter estimation framework that overcomes common limitations of traditional statistical approaches, such as strong distributional assumptions and the need for dense time-series data. By leveraging a synthetic data, we accurately recover all parameters of the model.

Our experiments demonstrate that the SCNN outperforms alternative methods (RNN, LSTM, and random forest) in parameter estimation. Under both (FSH) and  (ESH), the SCNN achieves the errors: 0.0189 and 0.0009 for the eight model parameters ($\mu, h, \eta, \beta, \delta, \lambda, \gamma, \alpha$)  and the two noise intensities $(\sigma_1, \sigma_2)$ under (FSH), and 0.0190 and 0.0009 under (ESH), respectively.

These findings contribute to a more realistic representation of forest transition processes and provide a practical framework for supporting evidence-based land management and policy development. Future research will explore the model's continuity dependence on parameters and investigate the existence of an invariant measure for the stochastic system.

\section*{Use of AI tools declaration}
The authors declare they have not used Artificial Intelligence (AI) tools in the creation of this article.

\section*{Acknowledgments}

The work of the first and last authors was supported by the WISE program (MEXT) at Kyushu University.

\section*{Conflict of interest}

{The authors declare that they have no conflict of interest.}



\begin{thebibliography}{999}
	
	\bibitem{Meyfroidt} P. Meyfroidt, E. F. Lambin, 
	Global forest transition: Prospects for an end to deforestation,  
	\emph{Annu. Rev. Environ. Resour.}, 
	\textbf{36} (2011), 343--371. https://doi.org/10.1146/annurev-environ-090710-143732
	
	\bibitem{Mather} A. S. Mather, 
	The forest transition, 
	\emph{Area}, 
	\textbf{24} (1992), 367--379.
	
	\bibitem{rudel2005} T. K. Rudel, O. T. Coomes, E. Moran, F. Achard, A. Angelsen, J. Xu, et al.,
	Forest transitions: Toward a global understanding of land use change, 
	\emph{Global Environ. Change}, 
	\textbf{15} (2005), 23--31. https://doi.org/10.1016/j.gloenvcha.2004.11.001
	
	\bibitem{Walker2004} R. Walker, S. A. Drzyzga, Y. Li, J. Qi, M. Caldas, E. Arima, et al., 
	A behavioral model of landscape change in the Amazon Basin: The colonist case, 
	\emph{Ecol. Appl.}, 
	\textbf{14} (2004), 299--312.  https://doi.org/10.1890/01-6004
	
	\bibitem{Lambin2010} E. F. Lambin, P. Meyfroidt, 
	Land use transitions: Socio-ecological feedback versus socio-economic change, 
	\emph{Land Use Policy}, 
	\textbf{27} (2010), 108--118. https://doi.org/10.1016/j.landusepol.2009.09.003
	
	\bibitem{Satake} A. Satake, T. K. Rudel, 
	Modeling the forest transition: Forest scarcity and ecosystem service hypotheses, 
	\emph{Ecol. Appl.}, 
	\textbf{17} (2007), 2024--2036. https://doi.org/10.1890/07-0283.1
	
	\bibitem{rudel2020} T. K. Rudel, P. Meyfroidt, R. Chazdon, F. Bongers, S. Sloan, H. R. Grau, et al., 
	Whither the forest transition? Climate change, policy responses, and redistributed forests in the twenty-first century, 
	\emph{Ambio}, 
	\textbf{49} (2020), 74--84. https://doi.org/10.1007/s13280-018-01143-0
	
	\bibitem{goni2018} I. Iriarte-Goñi, M. I. Ayuda, 
	Should forest transition theory include effects on forest fires? The case of Spain in the second half of the twentieth century, 
	\emph{Land Use Policy}, 
	\textbf{76} (2018), 789--797. https://doi.org/10.1016/j.landusepol.2018.03.009
	
	\bibitem{Allen2003} L. J. S. Allen, 
	\emph{An Introduction to Stochastic Processes with Applications to Biology}, 
	Pearson Education, 2003.
	
	\bibitem{Mao} X. Mao, 
	\emph{Stochastic Differential Equations and Applications}, 
	Horwood Publishing, 2008.
	
	\bibitem{Gao2025} Y. Gao, M. Banerjee, V. T. Ta, 
	Dynamics of infectious diseases in predator–prey populations: A stochastic model, sustainability, and invariant measure, 
	\emph{Math. Comput. Simul.}, 
	\textbf{227} (2025), 103--120. https://doi.org/10.1016/j.matcom.2024.07.031
	
	\bibitem{Aditya2024} D. A. Hartono, T. H. L. Nguyen, V. T. Ta, 
	A stochastic differential equation model for predator-avoidance fish schooling, 
	\emph{Math. Biosci.}, 
	\textbf{367} (2024), 109112. https://doi.org/10.1016/j.mbs.2023.109112
	
	\bibitem{Raissi2019} M. Raissi, P. Perdikaris, G. E. Karniadakis, 
	Physics-informed neural networks: A deep learning framework for solving forward and inverse problems involving nonlinear partial differential equations, 
	\emph{J. Comput. Phys.}, 
	\textbf{378} (2019), 686--707. https://doi.org/10.1016/j.jcp.2018.10.045
	
	\bibitem{Friedman} A. Friedman, 
	\emph{Stochastic Differential Equations and Applications}, 
	Dover Publications, 2006.
	
	\bibitem{Arnold} L. Arnold, 
	\emph{Stochastic Differential Equations: Theory and Applications}, 
	Wiley-Interscience, 1974.
	
	\bibitem{ta2015sustainability} V. T. Ta, T. H. L. Nguyen, A. Yagi, 
	A sustainability condition for stochastic forest model, 
	\emph{Commun. Pure Appl. Anal.}, 
	\textbf{16} (2017), 699--718. https://doi.org/10.3934/cpaa.2017034
	
	\bibitem{Kloeden2003} P. E. Kloeden, E. Platen, H. Schurz, 
	\emph{Numerical Solution of SDE Through Computer Experiments}, 
	Springer, 2003.
	
	\bibitem{d2l} A. Zhang, Z. C. Lipton, M. Li, A. J. Smola, 
	\emph{Dive into Deep Learning}, 
	Cambridge University Press, 2023.

\end{thebibliography}
\end{document}